\documentclass[10pt,twocolumn]{article}
\usepackage[
left=2.00cm,
right=2.00cm,
top=2.00cm,
bottom=2.00cm
]{geometry}
\usepackage{titlesec}
\titleformat{\section}[hang]
{\large\bfseries}
{\thesection.}{0.5em}{}

\titleformat{\subsection}[hang]
{\bfseries}
{\thesubsection.}{0.5em}{}

\usepackage{footnote}
\makesavenoteenv{tabular}
\makesavenoteenv{table}

\usepackage{amsmath, amssymb, amsthm}
\usepackage{mathtools}
\usepackage{algorithm, algorithmic, multicol}
\usepackage{xcolor, soul}
\usepackage[title]{appendix}
\usepackage{amsthm}
\usepackage{booktabs}
\usepackage{subfig}
\usepackage{graphicx}
\usepackage{grffile}
\usepackage{url,hyperref}
\usepackage[round]{natbib}
\newtheorem{theorem}{Theorem}[section]

\newtheorem{lemma}[theorem]{Lemma}
\newtheorem{definition}{Definition}[section]

\newcommand{\Ind}{\mathbb{I}}
\newcommand{\ytilde}{\tilde{y}}
\newcommand{\yhat}{\hat{y}}
\newcommand{\E}{\mathbb{E}}
\newcommand{\Prob}{\mathbb{P}}

\newcommand{\chat}{\hat{c}}
\newcommand{\fhat}{\hat{f}}

\newcommand\reals{\mathbb{R}}
\newcommand\onev{\textbf{1}}
\newcommand\zerov{\textbf{0}}
\newcommand\zeroOne{l^{0-1}}
\newcommand\zeroOneHat{\hat{l}^{0-1}}

\newcommand\stdv{e}
\newcommand\costsOriginal{\mathcal{C}^{eor}_1}

\newcommand\weakspace{\mathcal{H}}
\newcommand\sumtt{\sum_{t=1}^T}
\newcommand\logLoss{l^{\log}}

\DeclareMathOperator{\argmax}{argmax}
\DeclareMathOperator{\argmin}{argmin}

\title{Online Multiclass Boosting with Bandit Feedback}

\author{
  Daniel T. Zhang\\
  Department of Computer Science\\
  University of Michigan\\
  \texttt{dtzhang@umich.edu} \\
  \and
  Young Hun Jung\\
  Department of Statistics\\
  University of Michigan\\
  \texttt{yhjung@umich.edu} \\
  \and
  Ambuj Tewari\\
  Department of Statistics\\
  University of Michigan\\
  \texttt{tewaria@umich.edu} \\
}

\begin{document}
\twocolumn[
\maketitle]
\begin{abstract}
We present online boosting algorithms for multiclass classification with bandit feedback, where the learner only receives feedback about the correctness of its prediction. We propose an unbiased estimate of the loss using a randomized prediction, allowing the model to update its weak learners with limited information. Using the unbiased estimate, we extend two full information boosting algorithms \citep{jmulti} to the bandit setting. We prove that the asymptotic error bounds of the bandit algorithms exactly match their full information counterparts. The cost of restricted feedback is reflected in the larger sample complexity. Experimental results also support our theoretical findings, and performance of the proposed models is comparable to that of an existing bandit boosting algorithm, which is limited to use binary weak learners
\end{abstract}

\section{INTRODUCTION}
We study the online multiclass classification problem with bandit feedback. In this setting, the data instances arrive sequentially, and the learner has to predict the label among a finite, but perhaps large, set of candidates. In certain practical settings, such as when the labels are ads or product recommendations on the web, the learner does not receive the correct label as feedback. Instead, it only receives feedback about whether its predicted label was correct (e.g., the user clicked on the ad or recommendation) or not (e.g., user did not click). However, training machine learning models under such partial feedback is challenging. A common approach is to convert a full information algorithm into a bandit version without incurring too much performance loss (see, for example, \cite{banditron} and \cite{beygelzimer2017efficient} for work using the perceptron algorithm). 

In this paper, we design online algorithms for multiclass classification under bandit feedback by building on recent online boosting work in the full-information setting. Online boosting algorithms combine the predictions of multiple online weak learners to improve prediction performance. Classical boosting algorithms were designed for the batch setting. \cite{chen2012online} and \cite{beygelzimer2015optimal} first developed a theory of online boosting for binary classification. Then \cite{jmulti} and \cite{jrank} extended the theory to the multiclass classification and the multilabel ranking problems. These works prove that the boosting algorithm's asymptotic error converges to zero if the number of weak learners, whose predictions are slightly better than random guessing, gets larger. 

Designing a boosting algorithm with bandit feedback is particularly difficult as it is not clear how to update the weak learners. For example, suppose that a weak learner $WL^1$ predicts the label $1$, another learner $WL^2$ predicts the label $2$, and the boosting algorithm predicts the label $1$, which turns out to be incorrect. We cannot even tell $WL^2$ whether its prediction is correct. To the best of our knowledge, \cite{binmulti} are the only ones who have proposed a boosting algorithm in the multiclass bandit setting. However, their algorithm is restricted to use binary weak learners and only updates a subset of them, viz. ones which can get full feedback. In contrast, our algorithms use multiclass weak learners and update every learner at each round.

To derive our algorithms and guarantees, we extend the work of \cite{jmulti} to the bandit setting. Instead of making a deterministic prediction, our algorithms randomize them. This allows them to estimate the loss using the distribution over labels, and this estimate is used to update the weak learners. Similar to the full information work, we propose a computationally expensive algorithm, BanditBBM, with an optimal error bound and a more practical algorithm, AdaBandit, with a suboptimal bound. AdaBandit is the first adaptive boosting algorithm in the bandit setting that does not assume weak learners' edge over random is known beforehand. Interestingly, our algorithms' asymptotic error bounds match the full information counterparts with increased sample complexity, which can be interpreted as the cost of bandit feedback.

\section{PRELIMINARIES}

We denote the indicator function by $\Ind(\cdot)$, the $i^{th}$ standard basis vector by $\stdv_i$, the vector of ones by $\onev$, and the vector of zeros by $\zerov$. We will use $[n]$ for the set $\{1, \cdots, n\}$, $\Delta_n$ for probability distributions over $[n]$. 

\vspace{-1mm}
\subsection{Problem Setting}
We first describe the online multiclass classification problem with bandit feedback. It is a sequential game between two players: learner and adversary. The set $[k]=\{1, \cdots, k\}$ of $k$ possible labels is known to both players. At each round $t=1, \cdots, T$, the adversary selects a labeled example $(x_t, y_t)\in\mathcal{X}\times [k]$ (where $\mathcal{X}$ is some domain) and sends only $x_t$ to the learner. The learner then tries to guess its label and sends its prediction $\hat{y}_t$ back to the adversary. As we are in the bandit setting, the adversary only reveals whether the prediction is correct by sending $\Ind(\hat{y}_t \neq y_t)$ to the learner. The learner's goal is to minimize the number of incorrect predictions. In other words, the learner's performance is evaluated by the zero-one loss (see (\ref{eq:zeroOne}) for definition). 

To tackle this problem, we use the online multiclass boosting setup of \cite{jmulti}. In this setting, the learner further splits into $N$ online weak learners, $WL^1, \cdots, WL^N$, as well as a booster that handles the weak learners. When the booster receives an unlabeled instance $x_t$, it shares this information with the weak learners and then aggregates their predictions to produce the final prediction. Here we assume that each weak learner $WL^{i}$ predicts a label $h^i_t$ in $[k]$. Once the booster gets the feedback from the adversary, it computes a cost vector $c^i_t \in \reals^k$ for $WL^i$ to incur the loss $c^i_{t, h^i_t}$ and update its prediction rule. It should be noted that even though our boosting algorithms are designed for the bandit setting, the weak learners observe full cost vectors $c^i_t \in \reals^k$, which are constructed from bandit feedback by the booster. 

\subsection{Unbiased Estimate of the Zero-One Loss}
Even though we have not specified how to compute cost vectors $c^i_t$, it is naturally expected that they should depend on the final zero-one loss vector:
\begin{equation}
\label{eq:zeroOne}
    \zeroOne_t = \onev - \stdv_{y_t} \in \reals^k.
\end{equation}
As we are in the bandit setting, the booster only has limited information about this vector. In particular, unless its final prediction is correct, only a single entry of $\zeroOne_t$ is available. 

A popular approach for algorithm design in the partial information setting is to obtain an unbiased estimate of the loss. To do so, many bandit algorithms randomize their prediction. In our setting, instead of making a deterministic prediction $\yhat$, the algorithm designs a sampling distribution $p_t \in \Delta_k$ as follows:
\begin{gather}
\label{pt_definition}
p_{t, i} =
\begin{cases}
1-\rho &\text{ if } i=\hat{y_t}\\
\frac{\rho}{k-1} &\text{ if } i \neq \hat{y_t}
\end{cases},
\end{gather}
where $\rho$ is a parameter that controls the exploration rate. This distribution puts a large weight on the label $\yhat_t$ and evenly distributes the remaining weight over the rest. The algorithm draws a final prediction $\ytilde_t$ based on $p_t$. In this way, the algorithm can build an estimator using the known sampling distribution. A simplest unbiased estimate of the zero-one loss is 
\begin{equation}
\label{eq:trivalEst}
    \zeroOneHat_t = \frac{\Ind(\ytilde_t = y_t)}{p_{t, \ytilde_t}}(\onev - \stdv_{\ytilde_t}) \in \reals^k.
\end{equation}
It is easy to check that this is indeed unbiased. However, it is not necessarily the best because it becomes a zero vector when the booster makes a mistake. As the zero loss vector does not provide any useful information, the weak learners cannot update at this round. Therefore, it would be hard for the booster to escape the early training stage using the simple estimate. 

As an alternative, we propose a new estimator
\begin{align}
\begin{split}
\label{lhat}
\zeroOneHat_{t, i}=&\frac{\Ind(\ytilde_t=y_t)}{p_{t, \ytilde_t}}\Ind(y_t\neq i)\Ind(\yhat_t\neq i)\\
&~~~~+ \frac{\Ind(\ytilde_t=\yhat_t)}{p_{t,\ytilde_t}}\Ind(\yhat_t\neq y_t)\Ind(\yhat_t=i).
\end{split}
\end{align}
We first emphasize that this quantity can be computed only using the bandit feedback. The proof that it is actually unbiased appears in Appendix \ref{section:unbiased}.

This estimator resolves the main issue with the estimator in~\eqref{eq:trivalEst}, viz. that the learner cannot update during a mistake round. In fact, it allows the weak learner to update on each instance with probability at least $1-\rho$. Furthermore, the algorithms using this estimator empirically performed much better than ones using the estimator in \eqref{eq:trivalEst}. For these reasons, we will stick to the estimate in \eqref{lhat} from now on. 

To apply concentration inequalities, we need to control the variance of estimators. We say a random vector $Y$ is \textit{$b$-bounded} if 
$||Y - \E Y||_\infty \leq b$ almost surely.
Note that this definition also applies to random variables (i.e., scalars), in which case the norm above simply becomes the absolute value.
It is easy to check our estimator $\zeroOneHat_t$ is $\frac{k}{\rho}$-bounded. 

Now suppose that a cost vector $c^i_t \in \reals^k$ (to be fed into weak learner $i$ at time $t$) requires the knowledge of the true label $y_t$. Since the label is usually unavailable, we also need to estimate the cost vector. We first compute a matrix $C^i_t \in \reals^{k\times k}$, whose $j^{th}$ column is the cost vector $c^i_t$ assuming $j$ is the correct label. Then we will use the following random cost vector:
\begin{equation}
\label{eq:costEst}
    \chat^i_t = C^i_t \cdot (\onev - \zeroOneHat_t ).
\end{equation}
Since $C^i_t$ is a deterministic matrix, we can compute
\begin{equation*}
    \E_{\ytilde_t} \chat^i_t 
    = C^i_t \cdot (\onev - \zeroOne_t) 
    = C^i_t \cdot \stdv_{y_t},
\end{equation*}
which is the $y_t^{th}$ column of $C^i_t$. This shows that $\chat^i_t$ is an unbiased estimate of $c^i_t$. 

\section{ALGORITHMS}
We introduce two different online boosting algorithms and provide their theoretical error bounds. As the booster's performance obviously depends on the weak learner's predictive power, we need a way to quantify the latter. Firstly, we define the edge of a weak learner over random guessing and assume every weak learner has a positive edge $\gamma$. This edge is closely related to the one defined in the full information setting, and hence we can easily compare the error bounds between the two settings. This idea leads to the algorithm BanditBBM, which has a very strong error bound. Secondly, instead of having an additional assumption on the weak learners, we measure empirical edges of the learners and use these quantities to bound the number of mistakes. We call this algorithm AdaBandit, and since it enjoys theoretical guarantees under fewer assumptions, it is more practical. 

\subsection{Algorithm Template}
\label{section:template}
Our boosting algorithms share a template, which we discuss here. As it adopts the cost vector framework from \cite{jmulti} and \cite{jrank}, the template is very similar except for the additional step of estimating the loss.

\begin{algorithm}[t]
	\begin{algorithmic}[1]
	    \STATE \textbf{Input:} Exploration rate $\rho$
		\STATE \textbf{Initialize:} Weak learner weights $\alpha^{i}_{1}$ for $i \in [N]$
		\FOR {$t = 1, \cdots, T$}
		\STATE Receive example $x_{t}$
		\STATE Get predictions $h^{i}_{t} \in [k]$ from $WL^i$ for $i \in [N]$
		\STATE Compute expert predictions for $j \in [N]$ \\
		$~~s^{j}_{t} = \sum_{i=1}^{j}\alpha^{i}_{t}\stdv_{h^{i}_{t}} \in \reals^k$ and 
		$\yhat^j_t = \argmax_l s^j_{t, l}$
		\STATE Choose an expert index $i_{t} \in [N]$ 
		\STATE Get an intermediate prediction $\yhat_t = \yhat^{i_t}_t$
		\STATE Compute $p_t$ in (\ref{pt_definition})
		\STATE Draw $\ytilde_t$ using $p_t$ and send to the adversary
		\STATE Receive feedback $\Ind(\ytilde_t \neq y_t)$
		\STATE Estimate the loss by $\zeroOneHat_t$ in (\ref{lhat})
		\STATE Update weights $\alpha^{i}_{t+1}$ for $i \in [N]$
		\STATE Compute cost vectors $\chat^{i}_{t}$ for $i \in [N]$
		\STATE Weak learners suffer the loss $\chat^{i}_{t, h^i_t}$
		\STATE Weak learners update the internal parameters
		\STATE Update the booster's parameters, if any
		\ENDFOR
	\end{algorithmic}
	\caption{Online Bandit Boosting Template}
	\label{alg:template}
\end{algorithm}

To keep the template in Algorithm \ref{alg:template} general, we do not specify certain steps, which will be finalized later for each algorithm. Also, we do not restrict weak learners in any way except requiring that each $WL^i$ predicts a label $h^i_t \in [k]$, receives a full loss vector $\chat^i_t \in \reals^k$, and suffers the loss $\chat^i_{t, h_t}$ according to their prediction.

The booster keeps updating the learner weights $\alpha^i_t$ and constructs experts. There are $N$ experts where the expert $j$ tracks the weighted cumulative votes among the first $j$ weak learners: $s^{j}_{t} = \sum_{i=1}^{j}\alpha^{i}_{t}\stdv_{h^{i}_{t}} \in \reals^k$. We also track $\yhat^j_t = \argmax_l s^j_{t, l}$, where the tie breaks arbitrarily. Then the booster chooses an expert index $i_t \in [N]$ at each round $t$ and decides the intermediate prediction $\yhat_t$ as $\yhat^{i_t}_t$. In other words, it takes a weighted majority vote among the first $i_t$ learners. BanditBBM fixes $i_t$ to be $N$, while AdaBandit draws it randomly using a calibrated distribution. Using $\yhat_t$ and the exploration rate $\rho$, the booster computes the sampling distribution $p_t \in \Delta_k$ as in (\ref{pt_definition}). A random label $\ytilde_t$ drawn from $p_t$ is the final prediction, and the booster gets the feedback $\Ind(\ytilde_t \neq y_t)$. Then it constructs the unbiased estimate of the zero-one loss in (\ref{lhat}) and updates the learner weights $\alpha^i_t$. Finally, it computes cost vectors $\chat^i_t\in\reals^k$ for $WL^i$ and lets them update their parameters.

\subsection{An Optimal Algorithm}
The first algorithm, BanditBBM (Bandit Boost-by-Majority), assumes the bandit weak learning condition, which states that weak learners are better than random guessing. The algorithm is optimal: it requires the minimal number of weak learners up to a constant factor to attain a certain accuracy. 

\subsubsection{Bandit Weak Learning Condition}
\label{bandit_WLC}
This section proposes a bandit weak learning condition which requires weak learners to do better than random guessing, having only observed unbiased estimates of cost vectors. At a high level, what the weak learning condition says is that, as far as the random cost vectors provided by the booster satisfy certain conditions, the weak learner can perform better than random guessing when graded against the expected cost vectors. As a baseline, we define $u^l_\gamma\in\Delta_k$ to be almost a uniform distribution that puts $\gamma$ more weight on the label $l$. As an example, $u^k_\gamma=(\frac{1-\gamma}{k}, \cdots, \frac{1-\gamma}{k}, \frac{1-\gamma}{k}+\gamma)$. The intuition is that if a learner predicts a label based on $u^{y_t}_\gamma$ at each round, then its accuracy would be better than random guessing by the edge $\gamma$. 

The booster's goal is to minimize the number of incorrect predictions and hence it wants to put the minimal cost on the correct label. In this regard, \cite{jmulti} constrain the choice of cost vectors\footnote{In fact, the authors constrain the choice of cost matrices, but it suffices to choose a specific row to get a cost vector.} to 
\begin{equation}
\label{eq:costOriginal}
    \costsOriginal = \{c \in \reals^k_+ ~|~ c_y = 0 \text{ and } ||c||_1 = 1\},
\end{equation}
where $y$ is the correct label. We allow a sample weight that can be multiplied by a cost vector to include scaled cost vectors. One remark is that we can always subtract a common number from every entry of the cost vector as we are interested in the relative loss. This means that as long as the booster puts the minimal cost on the correct label, one can transform it to represent as $wc$ for some weight $w$ and $c \in \costsOriginal$. Meanwhile, we are in the bandit setting, where the true label is often unavailable to the booster. Therefore, we allow our booster to compute a random vector $\chat^i_t$, whose expectation lies in $\costsOriginal$. Once the exploration rate $\rho$ is specified, we can additionally ensure that the random cost vectors are $\frac{k}{\rho}$-bounded.

It would be theoretically most sound if the bandit weak learning condition can be closely related to its full information counterpart. To do so, we present two weak learning conditions together. The settings are almost identical except the full information version observes a deterministic cost vector $c_t \in \costsOriginal$ while the bandit version only observes a randomized vector $\chat_t$ such that $\E_{\ytilde_t} \chat_t \in \costsOriginal$. Recall that the entire cost vector is shown to the learner even in the bandit setting. For both conditions, the time horizon is $T$, labeled data are chosen adaptively, and the parameters $\gamma, \delta$, and the sample weights $w_t$ lie in $[0, 1]$. 

\begin{definition}[OnlineWLC from \cite{jmulti}]
\label{def:fullWLC}
A pair of a learner and an adversary satisfies OnlineWLC($\gamma, \delta, S$) if the learner can generate predictions $\yhat_t$ such that we have with probability $1-\delta$, 
\begin{equation*}
    \sum_{t=1}^T w_t c_{t, \yhat_t} \leq \sum_{t=1}^T w_t c_t \cdot u^{y_t}_\gamma + S.
\end{equation*}
\end{definition}

\begin{definition}[BanditWLC]
\label{def:banditWLC}
Suppose the random cost vectors $\chat_t$ are $b$-bounded for some $b$. A pair of learner and adversary satisfies BanditWLC($\gamma, \delta, S$) if the learner can generate predictions $\yhat_t$, observing the random cost vectors $w_t \chat_t$, such that we have with probability $1-\delta$, 
\begin{equation*}
    \sum_{t=1}^T w_t c_{t, \yhat_t} \leq \sum_{t=1}^T w_t c_t \cdot u^{y_t}_\gamma + S,
\end{equation*}
where $c_t = \E \chat_t$ for all $t$. 
\end{definition}
Here $S$ is called excess loss. OnlineWLC is a special case of BanditWLC where the bound $b$ is $0$. In fact, we can show more intrinsic relations between the two. 

Suppose there is a fixed hypothesis class $\weakspace$ and an online learner makes a prediction $h_t(x_t)$ at time $t$ by choosing a hypothesis $h_t \in \weakspace$. Obviously, this setting does not cover all online learners, but the most widely used learners can be interpreted in this manner. \cite{jmulti} showed that the OnlineWLC can be derived from the following two assumptions:

\begin{itemize}
	\item (Online Richness Condition) For any sequence of cost vectors $(w_t, c_t) \in [0, 1] \times \costsOriginal$, there is a hypothesis $h\in \weakspace$ such that
	\begin{equation*}
	    \sum_{t=1}^T w_t c_{t, h(x_t)} \leq \sum_{t=1}^T w_t c_t \cdot u^{y_t}_\gamma.
	\end{equation*}
	
	\item (Online Agnostic Learnability Condition) For any sequence of (bounded) loss vectors $l_t \in \reals^k$, there is an online algorithm which can generate predictions $\yhat_t$ such that with probability $1-\delta$,
		\begin{gather*}
			\label{agnostic}
			\sum_{t=1}^T l_{t,\yhat_t} \leq \inf_{h \in \mathcal{H}} 
				\sum_{t=1}^T l_{t,h(x_t)} + R_\delta (T),
		\end{gather*}
		where $R_\delta(\cdot)$ is a sublinear regret.
\end{itemize}
Note that the online learnability condition only assumes a bounded loss instead of $\costsOriginal$. This condition holds, for example, if the space $\weakspace$ has a finite Littlestone dimension. Interested readers can refer the paper by \cite{daniely2015multiclass}. We show that these two conditions also imply the BanditWLC. The proof appears in Appendix \ref{boundingLemma}.

\begin{theorem}
\label{bridge}
Suppose a pair of weak learning space $\weakspace$ and adversary satisfies the richness condition with edge $2\gamma$ and the agnostic learnability condition with regret $R_\delta$(T). Addionally, we assume that $w_t \geq m$ for all $t$. Then the online learner based on $\weakspace$ satisfies BanditWLC($\gamma, 2\delta, S$) with
\begin{equation*}
    S=\sup_T -\frac{\gamma}{k}mT + b\sqrt{2T\log\frac{1}{\delta}} + R_\delta(T),
\end{equation*}
where $b$ is the bound of the random cost vectors. 
\end{theorem}
The extra condition $w_t \geq m$ is acceptable because if $w_t = 0$ for some $t$, then we can simply ignore this round because any prediction does not incur a loss. The excess loss $S$ is always finite due to the sublinear regret $R_\delta(T)$. Furthermore, a smaller $\delta$ would require a larger $S$. The excess loss in the BanditWLC is larger than the one in the OnlineWLC due to the term $b\sqrt{2T\log\frac{1}{\delta}}$. This is intuitive in that the learner needs more samples if only bandit feedback is available. Finally, the exploration rate $\rho$ also affects $S$ because $b$ is equal to $\frac{k}{\rho}$. This provides the following rough bound:
\begin{equation}
\label{eq:SBound}
    S = \tilde{O}(\frac{k}{\rho}),
\end{equation}
where $\tilde{O}$ suppresses dependence on $\log \frac{1}{\delta}$.

\subsubsection{BanditBBM Details}
\label{zero_one_direct}

Throughout the section, we assume the weak learners satisfy BanditWLC($\gamma, \delta, S$). BanditBBM is a modification of OnlineMBBM from \cite{jmulti} by incorporating the unbiased estimate of the loss in (\ref{lhat}). 

We use the potential function $\phi^y_i(s)$, discussed thoroughly in relation to boosting by \cite{offmult}, to design the cost vectors. The potential function $\phi^y_i$ takes the current cumulative votes $s \in \reals^k$ as an input and estimates the booster's loss when the true label is $y$ and there are $i$ weak learners left until the final prediction. In particular, it can be recursively defined as follows:
\begin{align*}
\phi^y_0(s)
&=\Ind(\argmax_i s_i \neq y)\\
\phi^y_{i+1}(s)
&=\E_{l\sim u^{y}_\gamma}\phi^y_i(s+\stdv_l).
\end{align*}
Unfortunately, this potential does not have a closed form. Since potential functions are the main ingredient to design cost vectors, their computation becomes a bottleneck when running the algorithm. This is a weakness of BanditBBM despite its strong mistake bound. However, one can use Monte Carlo simulations to approximate its value. 

Returning to our algorithm, we essentially want to set the cost vector to 
\begin{gather}
\label{eq:costBMBM}
    c^i_{t, l} = \phi^{y_t}_{N-i}(s^{i-1}_t + \stdv_l). 
\end{gather}
\cite{jmulti} prove that this cost vector puts the minimal cost on the correct label and thus it is a valid choice. The booster in our setting, however, cannot compute this vector as it requires the knowledge of the true label $y_t$. As an alternative, we create the following cost matrix
\begin{equation}
\label{eq:costMatBMBM}
    C^i_t[l, r] = \phi^{r}_{N-i}(s^{i-1}_t + \stdv_l)    
\end{equation}
and use (\ref{eq:costEst}) to compute a random cost vector $\chat^i_t$, which is an unbiased estimate of $c^i_t$ in (\ref{eq:costBMBM}).

\begin{algorithm}[t]
	\begin{algorithmic}[1]
	   % \STATE \textbf{Input:} Exploration rate $\rho$
        \setcounter{ALC@line}{1}	   
		\STATE \textbf{Initialize:} Set $\alpha^{i}_{1}=1$ for $i \in [N]$
% 		\FOR {$t = 1, \cdots, T$}
% 		\STATE Receive example $x_{t}$
% 		\STATE Get predictions $h^{i}_{t} \in [k]$ from $WL^i$ for $i \in [N]$
% 		\STATE Compute expert predictions \\
% 		$~~~~s^{j}_{t} = \sum_{i=1}^{j}\alpha^{i}_{t}\stdv_{h^{i}_{t}} \in \reals^k$ for $j \in [N]$
        \setcounter{ALC@line}{6}	   
		\STATE Set $i_t = N$
% 		\STATE Get an intermedidate prediction \\
% 		$~~~~\yhat_{t}=\argmax_j s^{i_{t}}_{t, j}$
% 		\STATE Compute $p_t$ in (\ref{pt_definition})
% 		\STATE Draw $\ytilde_t$ using $p_t$ and send to the adversary
% 		\STATE Receive feedback $\Ind(\ytilde_t \neq y_t)$
% 		\STATE Estimate the loss by $\zeroOneHat_t$ in (\ref{lhat})
        \setcounter{ALC@line}{12}
		\STATE Keep weights $\alpha^{i}_{t+1}=1$ for $i \in [N]$
		\STATE Compute cost vectors $\chat^{i}_{t}$ using \eqref{eq:costEst} and \eqref{eq:costMatBMBM}
% 		\STATE Weak learners suffer the loss $\chat^{i}_{t, h^i_t}$
% 		\STATE Weak learners update the internal parameters
        \setcounter{ALC@line}{16}
		\STATE There is no extra parameter
% 		\ENDFOR
	\end{algorithmic}
	\caption{BanditBBM Details}
	\label{alg:BMBM}
\end{algorithm}

The rest of the algorithm is straightforward. We set all weights $\alpha^i_t$ to be one and always choose the last expert: $i_t = N$. This means that the intermediate prediction $\yhat_t$ is a simple majority vote among all the weak learners. The reasoning behind this is that the booster wants to include all learners as they are strictly better than random, and all weak learners are equivalent in that they share the same edge $\gamma$. Algorithm \ref{alg:BMBM} summarizes the specifications. 

\subsubsection{Mistake Bound of BanditBBM}
We still assume that our weak learners satisfy BanditWLC($\gamma, \delta, S$). From observation (\ref{eq:SBound}), it is reasonable to assume $S = \tilde{O}(\frac{k}{\rho})$. Upon these assumptions, we can bound the number of mistakes made by BanditBBM. The proof appears in Appendix \ref{banditboost_proof}

\begin{theorem}[Mistake Bound of BanditBBM]
\label{thm:BMBM}
For any $T$, $N$ satisfying $\delta \ll \frac{1}{N}$,
the number of mistakes made by BanditBBM satisfies the following inequality with probability at least $1-(N+1) \delta$:
\begin{equation*}
    \sumtt \Ind(\ytilde_t \neq y_t) 
    \leq
    (k-1) e^{-\frac{\gamma^2 N}{2}} T +2\rho T + \tilde{O}(\frac{k^{7/2}\sqrt{N}}{\rho}),
\end{equation*}
where $\tilde{O}$
suppresses dependence on $\log \frac{1}{\delta}$.
\end{theorem}

If we set the exploration rate $\rho = \frac{k^{7/4}N^{1/4}}{\sqrt T}$, then the bound becomes 
\begin{equation*}
    (k-1) e^{-\frac{\gamma^2 N}{2}} T + \tilde{O}(k^{7/4}N^{1/4}\sqrt T).
\end{equation*}
Dividing by $T$, we can infer that $(k-1) e^{-\frac{\gamma^2 N}{2}}$ is the asymptotic error bound of the algorithm. This bound matches the bound of the full information counterpart, OnlineMBBM. Since it depends exponentially on $N$, BanditBBM does not require too many weak learners to obtain a desired accuracy. \cite{jmulti} also provide a lower bound in the full information setting, which shows that the exponential decay is the fastest rate one can expect for the asymptotic error bound. This result applies to our bandit setting as it is harder.  

\subsection{An Adaptive Algorithm}
While BanditBBM is theoretically sound, in real applications it has a number of drawbacks. Firstly, it is hard to identify the edge $\gamma$ of each weak learner, leading to incorrect computations of the potential function. Also, each learner may have a different edge, and assuming a common edge can underestimate some weak learner's predictive power. Finally, as pointed out in the previous section, evaluating the potential function is computationally expensive, which makes BanditBBM less useful in practice. To address these issues, we propose an adaptive algorithm, AdaBandit, based on the full information adaptive algorithm, Adaboost.OLM by \cite{jmulti}. Using the idea of improper learning, \cite{foster2018logistic} proposed another adaptive boosting algorithm that has a tighter sample complexity than Adaboost.OLM. However, we stick to Adaboost.OLM as it has the competitive asymptotic error bound, which is of primary interest in this paper. 

\subsubsection{Logistic Loss and Empirical Edges}
Instead of directly minimizing the zero-one loss, the adaptive algorithm tries to minimize a surrogate loss. As in Adaboost.OLM, we choose the following logistic loss $\logLoss_y:\reals^k \rightarrow \reals$:
\begin{gather*}
\logLoss_y(s)=\sum_{l=1}^k\log(1+\exp(s_l - s_{y})),
\end{gather*}
where $s$ is the cumulative votes of a chosen expert. 

As for the zero-one loss, computing the loss requires knowledge of the true label, and we again use the idea in (\ref{eq:costEst}) to estimate the loss. We want to emphasize that the logistic loss only plays an intermediate role in training, and the learner's predictions are still evaluated by the zero-one loss. 

Essentially, we want to set the cost vector $c^i_t = \nabla \logLoss_{y_t}(s^{i-1}_t)$. Since this depends on the true label $y_t$, we build a cost matrix $C^i_t \in \reals^{k \times k}$ as below:
\begin{align}
\begin{split}
\label{eq:adaCost}
C^i_{t}[l, r] = 
\begin{cases}
\frac{1}{1+\exp(s^{i-1}_{t, r}-s^{i-1}_{t, l})} & \text{if }l\neq r \\
-\sum_{j\neq r}\frac{1}{1+\exp(s^{i-1}_{t, r}-s^{i-1}_{t, j})} & \text{if }l= r
\end{cases}.
\end{split}
\end{align}
Note that each column also puts the minimal cost on the correct label $r$. Moreover, the sum of entries equals zero. Using the idea described in \eqref{eq:costEst}, we can compute $\chat^i_t$, which is an unbiased estimate of $c^i_t = \nabla \logLoss_{y_t}(s^{i-1}_t)$. 

Even though the adaptive algorithm does not assume the BanditWLC, we still need to measure the weak learners' predictive powers to analyze the booster's performance. As in the full information case, we use the following \textit{empirical edge} of $WL^i$:
\begin{equation*}
    \gamma_i =
% 	\frac{\sumtt c^i_{t, h^i_t}}{\sumtt c^i_{t, y_t}}.
	\sumtt c^i_{t, h^i_t} / \sumtt c^i_{t, y_t}.
\end{equation*}
Having the same empirical edge as Adaboost.OLM allows us to precisely evaluate the cost of bandit feedback. Based on our design of cost vector $c^i_t = \nabla \logLoss_{y_t}(s^{i-1}_t)$, we can check that $\gamma_i$ is in $[-1, 1]$ and a larger value implies a better accuracy. Obviously, the empirical edge is unavailable to the learner as it requires the true cost vector $c^i_t$. This is fine because we only use this value to provide the mistake bound. Running AdaBandit does not require the knowledge of empirical edges. 

\subsubsection{AdaBandit Details}
Now we describe the details of AdaBandit (see Algorithm \ref{alg:OLBM}). The choice of cost vectors $\chat^i_t$ is already discussed in the previous section. As this is an adaptive algorithm, we update the learner weights $\alpha^i_t$ to give more influence to high-performing learners. We also allow negative weights in case a weak learner is worse than random. 

As AdaBandit incorporates the logistic loss as a surrogate, we want to pick $\alpha^i_t$ to minimize 
\begin{equation*}
    \sumtt f^i_t(\alpha^i_t) \text{ where } f^i_t(\alpha) = \logLoss_{y_t}(s^{i-1}_t + \alpha \stdv_{h^i_t}),
\end{equation*}
where only the following unbiased estimator $\fhat^i_t$ is available to the learner:
\begin{equation*}
    \fhat^i_t(\alpha) = \sum_{j=1}^k \logLoss_j(s^{i-1}_t + \alpha \stdv_{h^i_t}) \cdot (1 - \zeroOneHat_{t, j}).
\end{equation*}
Since the logistic loss is convex, it is a classical online convex optimization problem, and we can use stochastic gradient descent (see \cite{zink} and \cite{shalev2014understanding}). Following the convention in Adaboost.OLM, we use the feasible set $F = [-2, 2]$ and the projection function $\Pi(\cdot) = \max\{-2, \min\{2, \cdot\}\}$ to update $\alpha^i_t$:
\begin{equation}
\label{eq:alphaUpdate}
    \alpha^i_{t+1} = \Pi(\alpha^i_t - \eta_t {\fhat^{i\prime}_t}(\alpha^i_t)),
\end{equation}
where $\eta_t$ is a learning rate. As the gradient of the logistic loss is universally bounded by $k$ and $\zeroOneHat_t$ is $\frac{k}{\rho}$-bounded, we can check that $|{\fhat^{i\prime}_t}(\alpha)| \leq \frac{2k^2}{\rho}$ almost surely. From this, if we set $\eta_t = \frac{\rho}{k^2\sqrt t}$, then a standard result in online stochastic gradient descent (see \cite{shalev2014understanding}, Chapter 14) provides with probability $1-\delta$,
\begin{equation}
\label{eq:convexRegret}
    \sumtt f^i_t (\alpha^i_t) 
    \leq
    \min_{\alpha \in [-2, 2]} \sumtt f^i_t(\alpha) + \tilde{O}(\frac{k^2}{\rho}\sqrt T),
\end{equation}
where $\tilde{O}$ suppresses dependence on $\log \frac{1}{\delta}$. 

We cannot prove that the last expert is the best because our weak learners do not adhere to the weak learning condition. Instead, we will show that at least one expert is reliable. To identify this expert, we use the \textit{Hedge algorithm} from \cite{hedge1} and \cite{hedge2}. This algorithm generally receives the zero-one loss of each expert. Since that is no longer available, we will feed $\zeroOneHat$, which in expectation reflects the true zero-one loss. As the exploration rate $\rho$ controls the variance of the loss estimate, we can combine the analysis of the Hedge algorithm with the concentration inequality to obtain a similar result. 

\begin{algorithm}[t]
	\begin{algorithmic}[1]
	   % \STATE \textbf{Input:} Exploration rate $\rho$
        \setcounter{ALC@line}{1}	   
		\STATE \textbf{Initialize:} Set $\alpha^{i}_{1}=0$ and $v^i_t=1$ for $i \in [N]$
% 		\FOR {$t = 1, \cdots, T$}
% 		\STATE Receive example $x_{t}$
% 		\STATE Get predictions $h^{i}_{t} \in [k]$ from $WL^i$ for $i \in [N]$
% 		\STATE Compute expert predictions \\
% 		$~~~~s^{j}_{t} = \sum_{i=1}^{j}\alpha^{i}_{t}\stdv_{h^{i}_{t}} \in \reals^k$ for $j \in [N]$
        \setcounter{ALC@line}{6}	   
		\STATE Randomly draw $i_t$ with $\Prob (i_t = i) \propto v^i_t$
% 		\STATE Get an intermedidate prediction \\
% 		$~~~~\yhat_{t}=\argmax_j s^{i_{t}}_{t, j}$
% 		\STATE Compute $p_t$ in (\ref{pt_definition})
% 		\STATE Draw $\ytilde_t$ using $p_t$ and send to the adversary
% 		\STATE Receive feedback $\Ind(\ytilde_t \neq y_t)$
% 		\STATE Estimate the loss by $\zeroOneHat_t$ in (\ref{lhat})
        \setcounter{ALC@line}{12}
		\STATE Update weights $\alpha^{i}_{t}$ using (\ref{eq:alphaUpdate}) with $\eta_t = \frac{\rho}{k^2\sqrt{t}}$
		\STATE Compute cost vectors $\chat^{i}_{t}$ using (\ref{eq:costEst}) and (\ref{eq:adaCost})
% 		\STATE Weak learners suffer the loss $\chat^{i}_{t, h^i_t}$
% 		\STATE Weak learners update the internal parameters
        \setcounter{ALC@line}{16}
		\STATE Update $v^i_{t+1} = v^i_t \cdot \exp(-\zeroOneHat_{t, \yhat^i_t})$
% 		\ENDFOR
	\end{algorithmic}
	\caption{AdaBandit Details}
	\label{alg:OLBM}
\end{algorithm}

\subsubsection{Mistake Bound of AdaBandit}
As mentioned earlier, we bound the number of mistakes made by the adaptive algorithm using the weak learners' empirical edges. We emphasize again that these empirical edges are defined exactly in the same manner with those used in the full information bound.

\begin{theorem}[Mistake Bound of AdaBandit]
\label{thm:OLBM}
For any $T$, $N$ satisfying $\delta \ll \frac{1}{N}$,
the number of mistakes made by AdaBandit satisfies the following inequality with probability at least $1-(N+4) \delta$:
\begin{equation*}
    \sumtt \Ind(\ytilde_t \neq y_t)
    \leq
    \frac{8k}{\sum_{i=1}^N \gamma_i^2}T
	+ 2\rho T + \tilde{O}(\frac{k^3N^2}{\rho^2\sum_{i=1}^N \gamma_i^2}),
\end{equation*}
where $\tilde{O}$
suppresses dependence on $\log \frac{1}{\delta}$.
\end{theorem}
If we set the exploration rate $\rho = \frac{kN^{2/3}}{(T\sum_{i=1}^N \gamma_i^2)^{1/3}}$, then the bound becomes 
\begin{equation*}
    \frac{8k}{\sum_{i=1}^N \gamma_i^2}T
    + \tilde{O}(\frac{kN^{\frac{2}{3}}}{(\sum_{i=1}^N \gamma_i^2)^{\frac{1}{3}}} T^{\frac{2}{3}}). 
\end{equation*}
This implies that $\frac{8k}{\sum_{i=1}^N \gamma_i^2}$ becomes the asymptotic error bound of AdaBandit, which matches the bound of Adaboost.OLM. \cite{jmulti} observe that $\gamma_i \geq \gamma$ with high probability if the learner has edge $\gamma$. Therefore, if our weak learners satisfy BanditWLC($\gamma, \delta, S$) as for BanditBBM, then the asymptotic bound becomes roughly $\frac{8k}{N\gamma^2}$. The bound depends polynomially on $N$, which is suboptimal. However, AdaBandit resolves the aforementioned issues of BanditBBM and actually shows comparable results on real data sets.

\section{EXPERIMENTS}
\label{experiments}
% \begin{table*}[t]
% \caption{Bandit and Full Information Asymptotic Performance} \label{asymptotic_table}
% \begin{center}
% \begin{tabular}{lccccccc}
% \toprule
% Data  &$k$ & StreamCnt  &OptBandit &AdaBandit &BinBandit &OptFull &AdaFull \\
% \midrule
% Balance         &3  &6250       &0.88       &0.89       &0.89       &0.88       &0.92\\
% Car             &4  &10368      &0.92       &0.91       &0.87       &0.96       &1.00\\
% Nursery         &4  &77760      &0.91       &0.94       &0.96       &0.99       &0.99\\
% Movement        &5  &165631     &0.90       &0.95       &0.91         &0.87       &0.98\\
% Mice            &8  &8640      &0.72       &0.89       &0.85       &0.86       &0.98\\
% Isolet          &26 &116955     &--         &--        &0.67       &0.99         &0.96\\
% \bottomrule
% \end{tabular}
% \end{center}
% \end{table*}

\begin{table*}[t]
\caption{Bandit and Full Information Asymptotic Performance} \label{asymptotic_table}
\begin{center}
\begin{tabular}{lccccccc}
\toprule
Data  &$k$ & StreamCnt  &OptBandit &AdaBandit &BinBandit &OptFull &AdaFull \\
\midrule
Balance         &3  &6250       &0.89   &0.97   &0.80   &0.76   &0.93\\
Car             &4  &10368      &0.88   &0.98   &0.84   &0.84   &0.96\\
Nursery         &4  &51840      &0.93   &0.95   &0.92   &0.94   &0.98\\
Movement        &5  &165631     &0.94   &0.97   &0.89   &0.92   &0.97\\
Mice            &8  &8640       &0.73   &0.87   &0.81   &0.81   &0.96\\
Isolet          &26 &116955     &0.48   &0.66   &0.66   &0.84  &0.90\\
\bottomrule
\end{tabular}
\end{center}
\end{table*}

We compare various boosting algorithms on benchmark data sets using publicly available code\footnote{ \url{https://github.com/pi224/banditboosting}}. The models include our proposed algorithms, BanditBBM and AdaBandit, their full information versions, OnlineMBBM and Adaboost.OLM from \cite{jmulti}, and BanditBoost from \cite{binmulti}. To maximize readability, we will call them by OptBandit, AdaBandit, OptFull, AdaFull, and BinBandit respectively, based on their characteristics. The first four models require multiclass weak learners, whereas BinBandit needs binary learners. For every model, we use online decision trees proposed by \cite{trees} as weak learners. 

We examine several data sets from the UCI data repository \citep{UCI1998,higuera2015self,ugulino2012wearable} that are tested by \cite{jmulti}. We follow the authors' data preprocessing to provide a consistent comparison. However, the bandit algorithms need more samples to reach their asymptotic performance. Because these data sets often have insufficient examples to yield this asymptotic performance, we duplicate and shuffle the data sets a number of times before feeding them to the algorithm. The amount of duplication done to each data set is chosen to suggest the asymptotic performance of each algorithm. Table \ref{asymptotic_table} contains a summary of data sets that are examined. The number of actual data points sent to each model is noted under the column StreamCnt. 

We optimize the number of weak learners $N$ for each bandit algorithm and data set, with granularity down to multiples of $5$. As BinBandit only takes binary weak learners, it needs more of them for data sets with large $k$. Thus, we use $10k$ weak learners for BinBandit on each data set. Recall that OptBandit and OptFull require the knowledge of the edge $\gamma$ from their weak learning condition. Since one cannot identify this value in practice, we also do not optimize this value and select $\gamma=0.1$ to be fixed. Lastly, the three bandit algorithms have the exploration rate $\rho$, which we optimize through the grid search and record the best results. A more detailed description of the experiment setting appears in Appendix \ref{experiments_appendix}.

\subsection{Asymptotic Performance}
\label{asymptotics}

Since the theoretical asymptotic error bounds of the proposed algorithms match their full information counterparts, we first compare the models' empirical asymptotic performance. To do so, we feed the first $80\%$ of the data without counting mistakes and compute the average accuracy on the last $20\%$ of the data. Table \ref{asymptotic_table} summarizes the results. The accuracy is averaged over 20 rounds for all data sets except Isolet and Movement, which we ran 10 times. These runs were computed with shuffling from 20 random seeds, a predetermined subset of which were used for Isolet and Movement. 

The full information algorithms exhibit very strong performance due to data duplication. Despite this, OptBandit and AdaBandit are quite competitive against them across data sets with smaller $k$. For datasets with larger $k$, our bandit algorithms do not keep up as well, as they receive less feedback per instance. Our algorithm's perform comparably to BinBandit, showing that our algorithms successfully combine the multiclass weak learners. A noteworthy aspect is how AdaBandit outperforms OptBandit on all the data sets, showing adaptive weighing's power. 
% It should be noted that BinBandit also updates the weights, which may help explain why it tends to perform better than OptBandit, and seems on par with AdaBandit.

\subsection{Analyzing Learning Curves}

% \vspace{-10mm}

Even though our bandit algorithms have the same asymptotic error bounds as their full information counterparts, the cost of bandit feedback is reflected in larger sample complexities. Investigating this, we compute approximate learning curves for these algorithms by recording the moving average accuracy across the latest \textit{$0.2 \times$(total rounds to be run)} data instances. For data sets of varying $k$ this illustrates the hardness of the bandit problem: as $k$ increases, learning an appropriately performing hypothesis takes longer, but is achievable nonetheless. 

\begin{figure}[t]
\centering
\includegraphics[width=0.9\columnwidth]{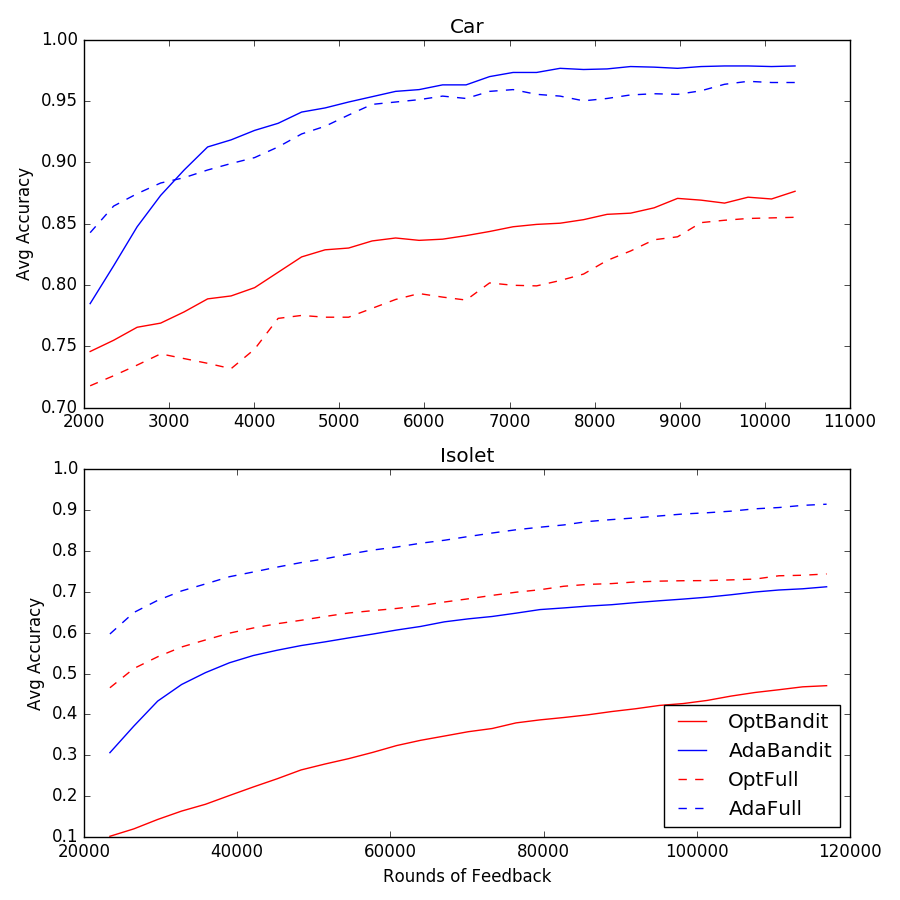}
\caption{Learning Curves on Car (Top) and Isolet (Bottom); Best Viewed in Color}
\label{fig1}
\end{figure}

Figure \ref{fig1} shows the learning curves on Car and Isolet data for our two bandit algorithms as compared with their full information counterparts. The curves for other data sets can be found in the Appendix \ref{experiments_appendix}. On Car data where $k$ is small, AdaBandit even outperforms OptFull and AdaFull by the end. This competitiveness with full information algorithms is reflected in the other learning curves in the appendix. On Isolet data with large $k$, our bandit methods lose some of their competitiveness. Given that the exploration rate $\rho$ is set to $0.1$ (whereas the theory would have it converge to $0$) and that the bandit algorithms have not fully plateaued off, the performance is still reasonable.

\subsubsection*{Acknowledgements}
AT acknowledges the support of the NSF CAREER grant IIS-1452099, a Sloan Research Fellowship, and the LSA Associate Professor Support Fund.

\bibliography{citations} 

\clearpage
\newpage

\begin{appendices}

\section{DETAILED PROOFS}
In this section, we include the full proofs that are omitted in the main manuscript. 

\subsection{Unbiased Estimate of The Zero-One Loss}
We prove the unbaisedness of our loss estimator presented in \eqref{lhat}.
\label{section:unbiased}
\begin{lemma}
\label{lemma:unbiased}
The estimator $\zeroOneHat_t$ in \eqref{lhat} is an unbiased estimator of the zero-one loss $\zeroOne_t$:
\begin{equation*}
    \E_{\ytilde_t \sim p_t} \zeroOneHat_t = \zeroOne_t.
\end{equation*}
\end{lemma}

\begin{proof}
Since $\ytilde$ is drawn with respect to $p_t$, we can write
\begin{align*}
    \E_{\ytilde_t \sim p_t} \zeroOneHat_{t, i}
    &= \Ind(y_t\neq i)\Ind(\yhat_t\neq i) 
    + \Ind(\yhat_t\neq y_t)\Ind(\yhat_t=i) \\
    &= \Ind(y_t\neq i)\Ind(\yhat_t\neq i) 
    + \Ind(i\neq y_t)\Ind(\yhat_t=i) \\
    &= \Ind(y_t\neq i)(\Ind(\yhat_t\neq i) + \Ind(\yhat_t=i)) \\
    &= \Ind(y_t\neq i),
\end{align*}
where the last term is $\zeroOne_{t, i}$, which completes the proof.
\end{proof}

\subsection{Proof of Theorem \ref{bridge}}
Since the agnostic learnability condition is given with deterministic cost vectors, we need to bridge the deterministic costs with the randomized ones. Observe that $\chat^i_t$ relies solely on the random draw of $\ytilde_t$ at each round. Therefore, the partial sum of random vectors $S_j = \sum_{t=1}^j \chat^i_t - c^i_t$ has the martingale property. Then we can prove the following lemma using the Azuma-Hoeffding inequality.

\begin{lemma}
\label{boundingLemma}
Suppose the random cost vectors $\chat_t$ are $b$-bounded. Let $p_t$ be a probability vector in $\Delta_k$. Then the following inequality holds with probability $1-\delta$:
\begin{equation*}
    |\sumtt (\chat_{t} - c_{t})\cdot p_t| 
    \leq b \sqrt{2T\log\frac{2}{\delta}}.
\end{equation*}
\end{lemma}

\begin{proof}
Since $\chat_t$ is $b$-bounded and unbiased, we have
\begin{equation*}
    |(\chat_{t} - c_{t})\cdot p_t| \leq b 
    \text{ a.s. and } 
    \E (\chat_{t} - c_{t})\cdot p_t = 0.
\end{equation*}
Therefore the Azuma-Hoeffding's inequality implies
\begin{align*}
    \Prob(|\sumtt (\chat_{t} - c_{t})\cdot p_t| \geq \epsilon)
    \leq 
    2 e^{-\frac{\epsilon^2}{2b^2 T}}.
\end{align*}
Putting $\epsilon = b \sqrt{2T\log\frac{2}{\delta}}$ finishes the proof.
\end{proof}

We now go into the main proof of Theorem \ref{bridge}.

\begin{proof}
Fix the sequence of cost vectors $w_t c_t$. From the richness condition with edge $2\gamma$, we know
\begin{equation*}
    \inf_{h \in \mathcal{H}} \sumtt w_t c_{t,h(x_t)} 
    \leq 
    \sumtt w_t c_t \cdot u^{y_t}_{2\gamma}.
\end{equation*}
By applying Lemma \ref{boundingLemma} with $p_t = \stdv_{h(x_t)}$, we get with probability $1-\delta$,
\begin{equation}
    \label{eq:bounding1}
    \inf_{h \in \mathcal{H}} \sumtt w_t \chat_{t,h(x_t)} 
    \leq 
    \sumtt w_t c_t \cdot u^{y_t}_{2\gamma} + b \sqrt{2T\log\frac{2}{\delta}}.
\end{equation}
% Again using the same lemma with $p_t = u^{y_t}_{2\gamma}$, we get with probability $1-2\delta$,
% \begin{equation}
%     \label{eq:bounding1}
%     \inf_{h \in \mathcal{H}} \sumtt w_t \chat_{t,h(x_t)} 
%     \leq 
%     \sumtt w_t \chat_t \cdot u^{y_t}_{2\gamma} + 2b \sqrt{2T\log\frac{2}{\delta}}.
% \end{equation}
Then by the online learnability condition, the online learner based on $\weakspace$ can generate predictions $\yhat_t$ that satisfies the following inequality with probability $1-\delta$: 
\begin{gather*}
	\sumtt w_t \chat_{t,\yhat_t} \leq \inf_{h \in \mathcal{H}} 
		\sumtt w_t \chat_{t,h(x_t)} + R_\delta (T),
\end{gather*}
Using (\ref{eq:bounding1}) and the union bound, we have with probability $1-2\delta$,
\begin{align*}
	\sumtt w_t \chat_{t,\yhat_t} 
	\leq 
	\sumtt w_t c_t \cdot u^{y_t}_{2\gamma} 
	+ b \sqrt{2T\log\frac{2}{\delta}} + R_\delta (T).
\end{align*}
Then by the definition of $\costsOriginal$ in (\ref{eq:costOriginal}), we can compute
\begin{equation*}
    c_t \cdot u^y_\gamma = \frac{1-\gamma}{k}.
\end{equation*}
Therefore, using the assumption $w_t \geq m$ for all $t$, we can bound
\begin{equation*}
    \sumtt w_t c_t \cdot (u^{y_t}_{\gamma} - u^{y_t}_{2\gamma}) 
    = \frac{\gamma}{k}\sumtt w_t
    \geq \frac{\gamma}{k} mT.
\end{equation*}
Then by taking 
\begin{equation*}
    S=\sup_T -\frac{\gamma}{k}mT + b\sqrt{2T\log\frac{1}{\delta}} + R_\delta(T),
\end{equation*}
we prove that with probability $1-2\delta$,
\begin{align*}
	\sumtt w_t \chat_{t,\yhat_t} 
	\leq 
	\sumtt w_t c_t \cdot u^{y_t}_{\gamma} 
	+ S,
\end{align*}
which shows the learner and the adversary satisfy BanditWLC($\gamma, 2\delta, S$). 
\end{proof}

\subsection{Proof of Theorem \ref{thm:BMBM}}
\label{banditboost_proof}
Note that the cost vectors defined in (\ref{eq:costBMBM}) does not put zero cost on the correct label. In order to apply the BanditWLC, we transform the cost vector. Since the zero-one loss vector has the minimal loss on the true label, we can inductively check that $\argmin_l c^i_{t,l} = y_t$. Then we define $d^i_t \in \reals^k$ as below:
\begin{equation*}
    d^i_{t, l} = c^i_{t, l} - c^i_{t, y_t}.
\end{equation*}
The minimal entry of $d^i_t$ is zero. Let $w^i_t = ||d^i_t||_1$, which plays a similar role of the sample weight in that $\frac{d^i_t}{w^i_t} \in \costsOriginal$. We also define $w^{i*} = \sup_t w^i_t$. 

We bound the cumulative potential functions by following the modified proof of Theorem 2 from \cite{jmulti}.

\begin{lemma}
With probability $1-N\delta$, we have 
\begin{equation*}
    \sumtt \phi^{y_t}_0(s^N_t)
    \leq
    \phi^1_N (\zerov) \cdot T + S \sum_{i=1}^N w^{i*}.
\end{equation*}
\end{lemma}

\begin{proof} In the proof of Theorem 2 by \cite{jmulti}, the authors write
\begin{align*}
\sumtt &\phi^{y_t}_{N-i+1}(s^{i-1}_t) \\
	&= \frac{1-\gamma}{k}\sumtt w^i_t
		- \sumtt d^i_{t, h^i_t} + \sumtt \phi^{y_t}_{N-i}(s^i_t).
% 	&= \frac{1-\gamma}{k}\norm{w^i}_1
% 		- \sum_t (c^i_t[l_t] - c^i_t[y_t]) + \sum_t \phi^{y_t}_{N-i}(s^i_t)
\end{align*}
Using the fact that our weak learners satisfy BanditWLC($\gamma, \delta, S$), we have with probability $1-\delta$
\begin{equation*}
    \frac{1}{w^{i*}}\sumtt d^i_{t, h^i_t}
    \leq 
    \frac{1-\gamma}{k w^{i*}} \sumtt w^i_t + S,
\end{equation*}
from which we deduce
\begin{align*}
    \sumtt &\phi^{y_t}_{N-i+1}(s^{i-1}_t) + w^{i*}S
    \geq 
    \sumtt \phi^{y_t}_{N-i}(s^i_t).
\end{align*}
Summing this over $i$ and using the union bound, we have with probability $1-N\delta$,
\begin{equation*}
    \sumtt \phi^{y_t}_0(s^N_t)
    \leq
    \sumtt \phi^{y_t}_N (\zerov) + S \sum_{i=1}^N w^{i*}.
\end{equation*}
By symmetry, we can check $\phi^{l}_N (\zerov) = \phi^{1}_N (\zerov)$ for any label $l \in [k]$, which completes the proof. 
\end{proof}
We now prove Theorem \ref{thm:BMBM}.
\begin{proof}
Since $\yhat_t = \argmax_l s^N_{t, l}$, we obtain 
\begin{equation*}
    \phi^{y_t}_0(s^N_t) 
    = 
    \Ind(\yhat_t \neq y_t).
\end{equation*}
Furthermore, \cite{jmulti} bound the terms that appear in the previous lemma:
\begin{gather*}
    \phi^{l}_N (\zerov) 
    \leq (k-1) e^{-\frac{\gamma^2 N}{2}} \\
    \sum_{i=1}^N w^{i*} 
    = O(k^{5/2}\sqrt N).
\end{gather*}
Combining these, we get with probability $1-N\delta$, 
\begin{align*}
    \sumtt \Ind(\yhat_t \neq y_t)
    &\leq
    (k-1) e^{-\frac{\gamma^2 N}{2}} T + O(k^{5/2}\sqrt NS)\\
    &\leq (k-1) e^{-\frac{\gamma^2 N}{2}} T + \tilde{O}(\frac{k^{7/2}\sqrt{N}}{\rho}),
\end{align*}
where the last inequality holds by (\ref{eq:SBound}).

To bound the booster's loss $\Ind(\ytilde_t \neq y_t)$, observe
\begin{equation*}
    \E_{\ytilde_t} \Ind(\ytilde_t \neq y_t) \leq \Ind(\yhat_t \neq y_t) + \rho.
\end{equation*}
Using the concentration inequality, we have with probability $1-(N+1)\delta$, 
\begin{align*}
    &\sumtt \Ind(\ytilde_t \neq y_t)\\
    &\leq
    (k-1) e^{-\frac{\gamma^2 N}{2}} T + \tilde{O}(\frac{k^{7/2}\sqrt{N}}{\rho})  + \rho T + \sqrt{T \log \frac{1}{\delta}}\\
    &\leq
    (k-1) e^{-\frac{\gamma^2 N}{2}} T +2\rho T + \tilde{O}(\frac{k^{7/2}\sqrt{N}}{\rho}),
\end{align*}
where we use the relation $\rho T + \frac{\log \frac{1}{\delta}}{\rho} \geq 2\sqrt {T\log \frac{1}{\delta}}$ to absorb the term $\sqrt{T \log \frac{1}{\delta}}$. This proves the main theorem. 
\end{proof}

\subsection{Proof of Theorem \ref{thm:OLBM}}
\label{adaptive_banditboost_proof}
We first recall a lemma from \cite{jmulti} to aid the proof.
\begin{lemma}[\cite{jmulti}, Lemma 11]
\label{opt_lemma}
Suppose $A, B \geq 0$, $B-A = \gamma \in [-1, 1]$, and $A+B \leq 1$. Then we have
\begin{gather*}
\min_{\alpha \in [-2,2]} A(e^{\alpha}-1) + B(e^{-\alpha}-1) \leq
	-\frac{\gamma^2}{2}.
\end{gather*}
\end{lemma}

Now we proceed with a bound of the zero-one loss of AdaBandit. The main structure of the proof results from the mistake bound of Adaboost.OLM by \cite{jmulti}.

\begin{proof}
We let $M_i$ denote the number of mistakes made by expert $i$: $M_i = \sumtt \Ind(\yhat^i_t \neq y_t )$. We also let $M_0 = T$ for convenience. As the booster uses the estimate $\zeroOneHat_t$ to run the Hedge algorithm, we define $\hat{M}_i = \sumtt \zeroOneHat_{t, \yhat^i_t}$ so that $\E_{\ytilde_1, \cdots, \ytilde_T} \hat{M}_i = M_i$. If we write $i^* = \argmin_i M_i$, then by the Azuma-Hoeffding inequality and the fact that $\zeroOneHat_t$ is $\frac{k}{\rho}$-bounded, we have with probability $1-\delta$,
\begin{equation*}
    \min_i \hat{M}_i \leq \hat{M}_{i^*} \leq \min_i M_i + \tilde{O}(\frac{k}{\rho}\sqrt{T}),
\end{equation*}
where $\tilde{O}$ suppresses dependence on 
$\log \frac{1}{\delta}$.

Then a standard analysis of the Hedge algorithm (see Corollary 2.3 by \cite{cesa2006prediction}) and the Azuma-Hoeffding inequality provide that with probability $1-3\delta$,
\begin{align}
\begin{split}
\label{hedge_bound}
    \sumtt \Ind(\yhat_t \neq y_t) 
    &\leq \sumtt \zeroOneHat_{t, \yhat_t} + \tilde{O}(\frac{k}{\rho}\sqrt{T})\\
    &\leq 2 \min_i \hat{M}_i + 2 \log N + \tilde{O}(\frac{k}{\rho}\sqrt{T})\\
    &\leq 2 \min_{i} M_i + 2\log N + \tilde{O}(\frac{k}{\rho}\sqrt{T}).
\end{split}
\end{align}
Now define $w^i = -\sumtt c^i_{t, y_t}$. If the expert $i-1$ makes a mistake at round $t$, there is $l \neq y_t$ such that $s^{i-1}_{t, y_t} \leq s^{i-1}_{t, l}$. According to (\ref{eq:adaCost}), this implies that $-c^i_{t, y_t} \geq \frac{1}{2}$. From this, we can deduce that 
\begin{equation}
    \label{eq:OLBMproof1}
    w^i \geq \frac{M_{i-1}}{2}.
\end{equation}
By our convention $M_0 = T$, the above inequality still holds for $i=1$. 

Next we define the difference in the cumulative logistic loss between two consecutive experts as
\begin{align*}
    \Delta_i 
    &= \sumtt\logLoss_{y_t}(s^{i}_t)
	- \logLoss_{y_t}(s^{i-1}_t)\\
    &= \sumtt\logLoss_{y_t}(s^{i-1}_t + \alpha^i_t e_{h^i_t})
	- \logLoss_{y_t}(s^{i-1}_t).
\end{align*}
From (\ref{eq:convexRegret}), we have with probability $1-\delta$,
\begin{align}
\begin{split}
\label{ogd_bound}
\Delta_i \leq \min_{\alpha \in [-2, 2]}
	\sumtt[\logLoss_{y_t}(s^{i-1}_t + \alpha e_{h^i_t})
	&- \logLoss_{y_t}(s^{i-1}_t)]\\
	&+ \tilde{O}(\frac{k^2}{\rho}\sqrt{T}).
\end{split}
\end{align}
Let us record an inequality:
\begin{align*}
    \log(1 + e^{s+\alpha}) - \log (1 + e^s) 
    &= \log(1 + \frac{e^\alpha - 1}{1 + e^{-s}}) \\
    &\leq \frac{e^\alpha -1}{1 + e^{-s}}.
\end{align*}
Using this, we can write 
\begin{align*}
    \logLoss_{y_t}(s^{i-1}_t + \alpha e_{h^i_t})
	- \logLoss_{y_t}(s^{i-1}_t)~~~~~~~~~~~~~~~~~~& \\
	\leq
	\begin{cases}
	    c^i_{t, h^i_t} (e^\alpha - 1) &\text{ if } h^i_t \neq y_t \\
	    c^i_{t, h^i_t} (-e^{-\alpha} + 1) &\text{ if } h^i_t = y_t
	\end{cases}.
\end{align*}
Summing this over $t$, we get 
\begin{align*}
    \sumtt\logLoss_{y_t}(s^{i-1}_t + \alpha e_{h^i_t})
	&- \logLoss_{y_t}(s^{i-1}_t)\\
	&\leq w^i(A(e^\alpha - 1) + B(e^{-\alpha}-1)),
\end{align*}
where 
\begin{align*}
    A=\sum_{t:h^i_t \neq y_t} c^i_{t, h^i_t}/w^i,~
    B=-\sum_{t:h^i_t = y_t} c^i_{t, h^i_t}/w^i.
\end{align*}
By (\ref{eq:adaCost}), $A$ and $B$ are non-negative and $B-A = \gamma_i \in [-1, 1]$, which is the empirical edge of $WL^i$. Then Lemma \ref{opt_lemma} implies 
\begin{align*}
    \min_{\alpha \in [-2, 2]}\sumtt\logLoss_{y_t}(s^{i-1}_t + \alpha e_{h^i_t})
	&- \logLoss_{y_t}(s^{i-1}_t)
	\leq -\frac{\gamma_i^2}{2}w^i.
\end{align*}
Combining this result with (\ref{eq:OLBMproof1}) and (\ref{ogd_bound}), we have with probability $1-\delta$, 
\begin{equation*}
    \Delta_i \leq -\frac{\gamma_i^2}{4}M_{i-1} + \tilde{O}(\frac{k^2}{\rho}\sqrt T).
\end{equation*}
Summing this over $i$ and using the union bound, we have with probability $1-N\delta$, 
\begin{align*}
    \sumtt \logLoss_{y_t}(s^N_t) 
    &- \logLoss_{y_t}(\zerov) \\
    &\leq -\frac{\min_i M_i}{4} \sum_{i=1}^N \gamma_i^2 + \tilde{O}(\frac{k^2N}{\rho}\sqrt T).
\end{align*}
Since $\logLoss_{y_t}(\zerov) = (k-1)\log 2$ and $\logLoss_{y_t}(s^N_t)\geq 0$, we have with probability $1-N\delta$, 
\begin{align*}
    \min_i M_i \leq 
    \frac{4(k-1) \log2}{\sum_{i=1}^N \gamma_i^2}T + \tilde{O}(\frac{k^2N}{\rho\sum_{i=1}^N \gamma_i^2}\sqrt T).
\end{align*}
Using this to (\ref{hedge_bound}), we get with probability $1-(N+3)\delta$, 
\begin{gather*}
    \sumtt \Ind(\yhat_t \neq y_t) 
    \leq
	\frac{8(k-1) \log2}{\sum_{i=1}^N \gamma_i^2}T
	+ \tilde{O}(\frac{k^2N}{\rho\sum_{i=1}^N \gamma_i^2}\sqrt T).
\end{gather*}
Then by the same argument as in Appendix \ref{banditboost_proof}, we get with probability $1-(N+4)\delta$,
\begin{align*}
    \sumtt& \Ind(\ytilde_t \neq y_t) \\
    &\leq
	\frac{8(k-1) \log2}{\sum_{i=1}^N \gamma_i^2}T
	+ 2\rho T + \tilde{O}(\frac{k^2N}{\rho\sum_{i=1}^N \gamma_i^2}\sqrt T)\\
	&\leq
	\frac{8k}{\sum_{i=1}^N \gamma_i^2}T
	+ 2\rho T + \tilde{O}(\frac{k^3N^2}{\rho^2\sum_{i=1}^N \gamma_i^2}),
\end{align*}
where the last inequality comes from the arithmetic mean and geometric mean relation: 
\begin{equation*}
    2ck^2N\sqrt T \leq kT + c^2 k^3 N^2.
\end{equation*}
\end{proof}

\section{DETAILED DESCRIPTIONS OF EXPERIMENTS}
\label{experiments_appendix}
% \begin{table*}[t]
% \caption{Bandit and Full Information Total Accuracy} \label{warmup_table}
% \begin{center}
% \begin{tabular}{lccccccc}
% \toprule
% Data  &$k$ &StreamCnt &OptBandit &AdaBandit &BinBandit &OptFull &AdaFull \\
% \midrule
% Balance         &3  &6250       &0.70       &0.78       &0.82       &0.71       &0.84\\
% Car             &4  &10368      &0.73       &0.84       &0.77       &0.78       &0.92\\
% Nursery         &4  &77760      &0.81       &0.90       &0.89       &0.90       &0.94\\
% Movement        &5  &165631     &0.78       &0.91       &0.83        &0.84       &0.95\\
% Mice            &8  &8640      &0.53       &0.62       &0.63       &0.66       &0.86\\
% Isolet          &26 &116955    &--         &--        &0.55       &--        &0.82\\
% \bottomrule
% \end{tabular}
% \end{center}
% \end{table*}

\begin{table*}[t]
\caption{Bandit and Full Information Total Accuracy} \label{warmup_table}
\begin{center}
\begin{tabular}{lccccccc}
\toprule
Data  &$k$ &StreamCnt &OptBandit &AdaBandit &BinBandit &OptFull &AdaFull \\
\midrule
Balance         &3  &6250       &0.83   &0.91   &0.73   &0.71   &0.85\\
Car             &4  &10368      &0.82   &0.93   &0.78   &0.80   &0.93\\
Nursery         &4  &51840      &0.89   &0.92   &0.89   &0.91   &0.95\\
Movement        &5  &165631     &0.89   &0.95   &0.84   &0.87   &0.95\\
Mice            &8  &8640       &0.53   &0.71   &0.62   &0.65   &0.84\\
Isolet          &26 &116955     &0.32   &0.51   &0.52   &0.74   &0.78\\
\bottomrule
\end{tabular}
\end{center}
\end{table*}

We discuss the experimental results more in detail.

\subsection{Data Set Details}
\label{dataset_info}

We modified the data sets identically as in \cite{jmulti}. In particular, we replaced missing data values in Mice data with $0$ and removed user information from Movement, leaving only sensor data in the latter. A single data point with missing values was also removed from Movement. Lastly, for Isolet the original $617$ covariates were projected onto the top $50$ principal components of the data set, retaining 80\% of the variance. Table \ref{table:dataSummary} summarizes the data information after preprocessing.

\begin{table}[b]
\caption{Data Set Summary} 
\begin{center}
\begin{tabular}{lccc}
\toprule
Data    &Size   &Dimension   &$k$\\
\midrule
Balance     &625    &4  &3 \\
Car         &1728   &82 &4 \\
Nursery     &12960  &4  &8 \\
Mice        &1080   &82 &8 \\
Isolet      &7797   &50 &26 \\
Movement    &165631 &12 &4 \\
\bottomrule
\end{tabular}
\end{center}
\label{table:dataSummary}
\end{table}

\subsection{Parameter Tuning}
\label{rho_experiments}
Our boosting algorithms have a few parameters: the number of weak learners $N$, the edge $\gamma$ for BanditBBM, and the exploration rate $\rho$. As the goal of the experiment is to compare the bandit algorithms with their full information counterparts, we did not optimize the parameters too hard. We optimized $N$ up to multiples of 5 and fix $\gamma=0.1$ for all data sets. 

The only parameter we tried to fit is exploration rate. To obtain a reasonable $\rho$, we ran a grid search keeping all other parameters stable and observing accuracy on a random stream from each data set. The chosen $\rho$ values reflect choices that made all the bandit algorithms perform well on each set. Table \ref{rho_table} shows the chosen $\rho$ from each of these grid searches. How many data points were streamed to obtain the final accuracy is shown in column Count.

\begin{table}[b]
\caption{Parameters for Bandit Algorithms} \label{rho_table}
\begin{center}
\begin{tabular}{lccccc}
\toprule
Data    &Count   &$\rho$   &$N_\text{Opt}$  &$N_\text{Ada}$ &$\gamma$ \\
\midrule
Balance     &2500   &0.001      &20     &15     &0.1\\
Car         &6912   &0.001      &15     &15     &0.1\\
Nursery     &12960   &0.001     &10     &5      &0.1\\
Movement    &165631  &0.001     &10     &20     &0.1\\
Mice        &4320   &0.1        &10     &20     &0.1\\
Isolet      &38985   &0.1       &10     &20     &0.1\\
\bottomrule
\end{tabular}
\end{center}
\end{table}

\subsection{Updating Weak Learners}
We used the VFDT algorithm designed by \cite{trees}. The learner takes a label and an importance weight to be updated. When a cost vector $\chat^i_t \in \reals^k$ is passed to the learner, we used 
\begin{equation*}
    l = \argmin_j \chat^i_{t, j} \text{ and } w = \sum_{j=1}^k (\chat^i_{t, j} - \chat^i_{t, l})
\end{equation*}
as the label and the importance weight, respectively. If there were multiple minima in $\chat^i_t$, we chose one of them randomly in a mistake round and selected the true label $y_t$ in a correct round if it minimized the cost vector.

One weakness of VFDT is that its performance is very sensitive to the range of importance weights. To address this, we added clipping in our implementation to prevent cost vectors from having excessively large entries. We introduced a magic number $100$ and clipped the entry whenever it went outside the range $[-100, 100]$. Clipping was especially helpful in stabilizing results for data sets with large $k$. Additionally, we scaled the importance weights for Isolet, as large $k$ tends to create larger weights.

\begin{figure}[b!]
\centering
\includegraphics[width=0.9\columnwidth]{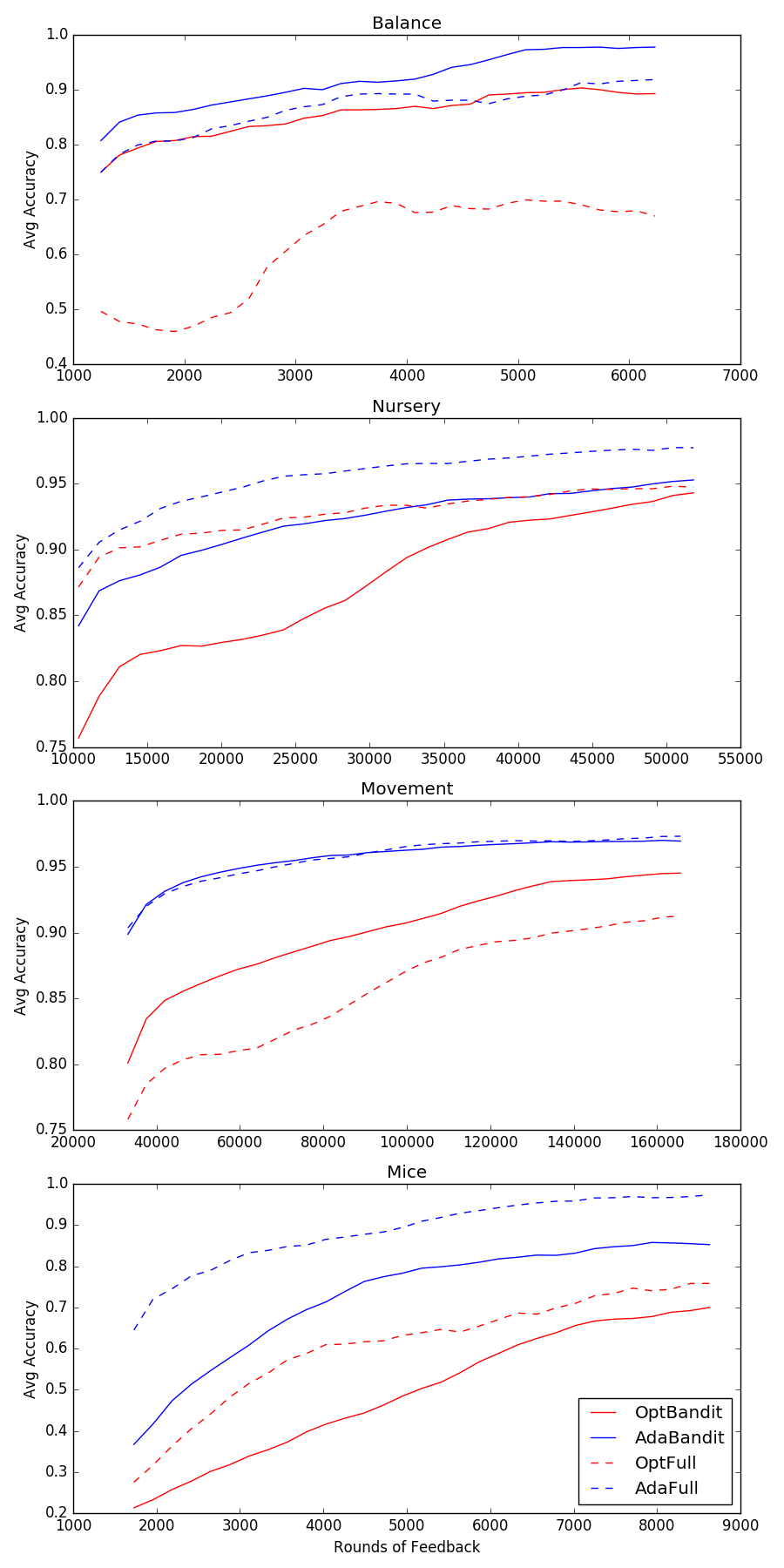}
\caption{Learning Curves on Balance (Top), Nursery, Movement, and Mice (Bottom); Best Viewed in Color}
\label{fig2}
\end{figure}

\subsection{Total Accuracy}
\label{warmup}

Table \ref{warmup_table} shows the average accuracy of all five algorithms on each of the data sets, in contrast to the asymptotic performance in Table \ref{asymptotic_table}. The effect of increased $k$ on the excess loss is noticeable, showing that the bandit algorithms learn less quickly. For data sets with smaller $k$, the total loss is a fairly large percentage of the asymptotic loss, indicating the the algorithms stay at their asymptotic accuracy for a larger fraction of rounds. For Mice and Isolet data, however, the total accuracy is significantly lower than the asympototic loss, indicating a much more linear improvement that is amortized over the whole data stream. Indeed, this corroborates Figure \ref{fig1}, where accuracy improvement of the bandit algorithms slows and tends towards a straight line for large $k$.

\subsection{Learning Curves}
Figure \ref{fig2} exhibits the learning curves of our bandit algorithms and the full information ones. Similar to the analysis in the main paper, the bandit algorithms learn slower than the full information algorithms in general, and the trend becomes obvious when $k$ gets larger. However, the bandit algorithms, especially AdaBandit, become competitive in the end and sometimes outperform the full information algorithms. The underperformance of the optimal algorithms is partially because we did not optimize the edge $\gamma$, and this aspect makes adaptive algorithms more suitable in practice. It should be noted that the learning curves do not begin at round $0$ because the window accuracy is not defined for a number of rounds less than 20\% of the total rounds to be given.

\end{appendices}

\end{document}